\documentclass[letterpaper, 10 pt, conference]{ieeeconf}  

\IEEEoverridecommandlockouts                              

\usepackage{amsmath} 

\usepackage{courier} 
\newcommand{\constant}[1]{{\ttfamily #1}}
\usepackage{tcolorbox}
\usepackage[utf8]{inputenc} 
\usepackage[T1]{fontenc}    
\usepackage{hyperref}       
\usepackage{url}            
\usepackage{booktabs}       
\usepackage{amsfonts}       
\usepackage{nicefrac}       
\usepackage{microtype}      
\usepackage{xcolor}         
\usepackage{mathrsfs}       
\usepackage{graphicx}

\usepackage{amssymb}

\usepackage{makecell}
\usepackage{multirow}

\usepackage{algorithm}
\usepackage{algorithmic}

\usepackage{tcolorbox}
\tcbuselibrary{breakable}

\newtheorem{proposition}{Proposition}

\usepackage{graphicx}
\usepackage{caption}
\usepackage{subcaption}
\usepackage{booktabs}
\usepackage{multirow}
\usepackage{booktabs}
\usepackage{multirow}
\usepackage{wrapfig}
\usepackage{graphicx}
\usepackage{array}
\newcommand{\StateCommentLabel}[3]{
\STATE #1 \label{#3} \hfill #2
}

\usepackage{hyperref}
\hypersetup{hypertex=true,
colorlinks=true,
linkcolor=blue,
anchorcolor=blue,
citecolor=blue}

\title{\LARGE \bf
HBTP: Heuristic Behavior Tree Planning with 
 Large Language Model  \\ Reasoning
}

\author{Yishuai Cai, Xinglin Chen$^{*}$, Yunxin Mao, Minglong Li$^{*}$, Shaowu Yang, Wenjing Yang and Ji Wang
\thanks{*Corresponding author}
\thanks{This research was supported by the National Natural Science Foundation of China (Grant Nos. 62106278, 62032024). }
\thanks{All the authors are with the College of
Computer Science and Technology, National University of Defense
Technology, Changsha 410073, China.
{ \tt\small \{caiyishuai, chenxinglin, liminglong10\}@nudt.edu.cn}}
}

\begin{document}

\maketitle
\thispagestyle{empty}
\pagestyle{empty}

\begin{abstract}






Behavior Trees (BTs) are increasingly becoming a popular control structure in robotics due to their modularity, reactivity, and robustness. In terms of BT generation methods, BT planning shows promise for generating reliable BTs. However, the scalability of BT planning is often constrained by prolonged planning times in complex scenarios, largely due to a lack of domain knowledge. In contrast, pre-trained Large Language Models (LLMs) have demonstrated task reasoning capabilities across various domains, though the correctness and safety of their planning remain uncertain. This paper proposes integrating BT planning with LLM reasoning, introducing Heuristic Behavior Tree Planning (HBTP)—a reliable and efficient framework for BT generation. The key idea in HBTP is to leverage LLMs for task-specific reasoning to generate a heuristic path, which BT planning can then follow to expand efficiently. We first introduce the heuristic BT expansion process, along with two heuristic variants designed for optimal planning and satisficing planning, respectively. Then, we propose methods to address the inaccuracies of LLM reasoning, including action space pruning and reflective feedback, to further enhance both reasoning accuracy and planning efficiency. Experiments demonstrate the theoretical bounds of HBTP, and results from four datasets confirm its practical effectiveness in everyday service robot applications.




\end{abstract}

\section{INTRODUCTION}

Behavior Trees (BTs) have become a widely adopted control architecture in robotics and embodied intelligence, offering advantages such as modularity, interpretability, and reactivity \cite{colledanchise2018behavior,ogren2022behavior,colledanchise2019learning}. BTs organize and control high-level behaviors using logical tree structures to achieve specific goals. Recently, BT planning has shown promise as an effective method for automatically generating reliable and robust BTs \cite{colledanchise2019blended,cai2021bt,chen2024integrating,cai2025mrbtp}. For example, BT Expansion \cite{cai2021bt} incrementally constructs BTs by exploring the action space, ensuring soundness and completeness in generating goal-directed behaviors. Similarly, \cite{chen2024integrating} introduced a two-stage framework where human instructions are translated into planning goals, followed by the Optimal BT Expansion Algorithm (OBTEA), which not only guarantees goal achievement but also minimizes the total cost of the action sequence. In this process, condition and action nodes can often be expressed using predicate logic, enhancing their clarity and interpretability. 


Nevertheless, a critical limitation of BT planning in real-world applications is the prolonged planning time, which primarily stems from the need to explore all possible actions due to the absence of domain-specific knowledge. As the complexity of the scenario increases, the number of actions and objects grows exponentially, further extending the process. However, we observe that in typical daily scenarios, usually only a small subset of relevant actions is required for a specific task. If this task-relevant subset can be identified in advance and prioritized in the search process, it would substantially reduce unnecessary exploration and lead to a more efficient planning process.


\begin{figure}[t]
\centering
\small
\includegraphics[width=\linewidth]{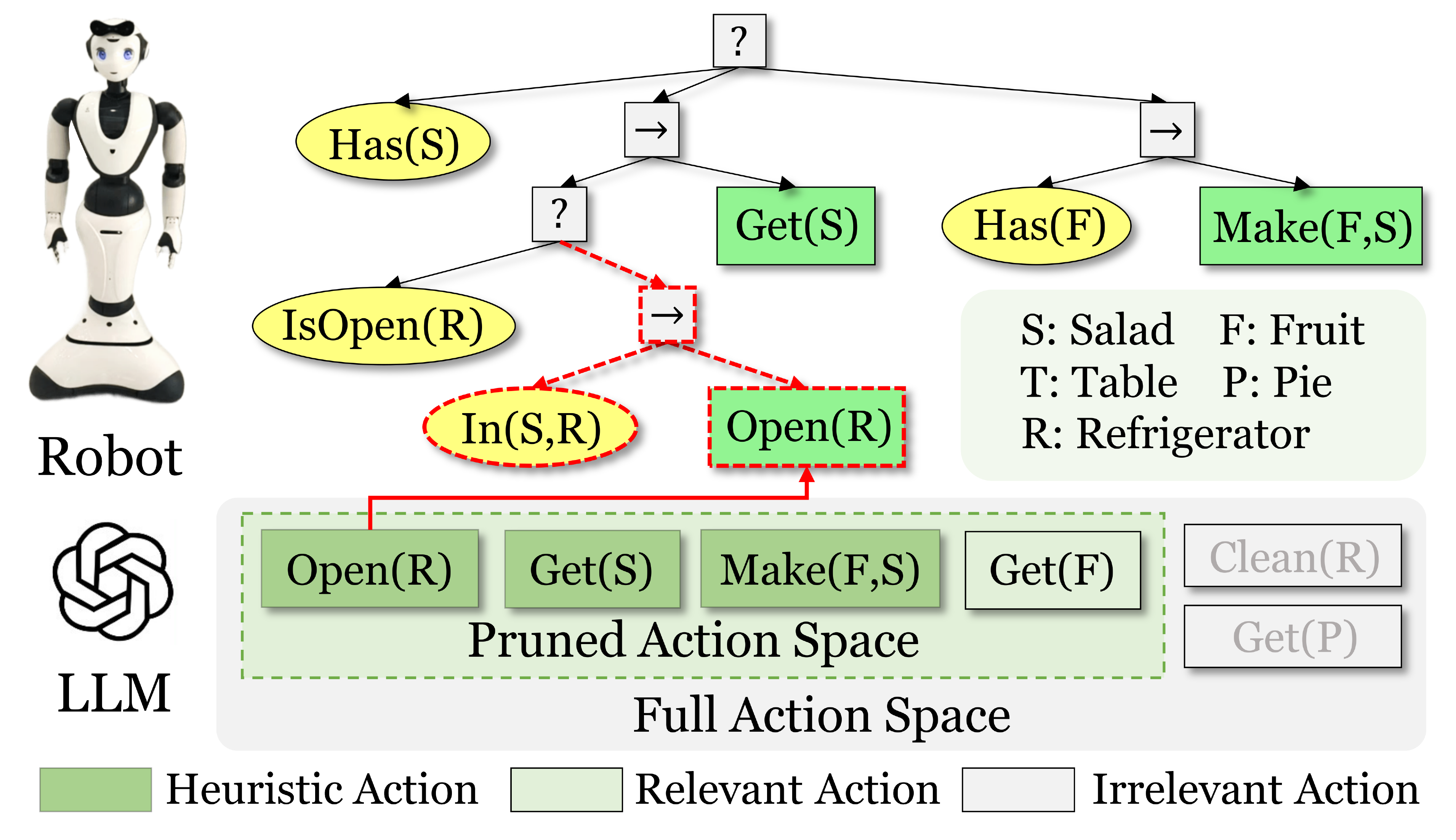}
\caption{The key concept of integrating BT planning with LLM reasoning. Given the task information, the LLM can divide the action space into heuristic actions, relevant actions, and irrelevant actions, effectively guiding heuristic BT planning.}
\label{fig:example}
\end{figure}

\begin{figure*}[t]
\centering
\small
\includegraphics[width=\linewidth]{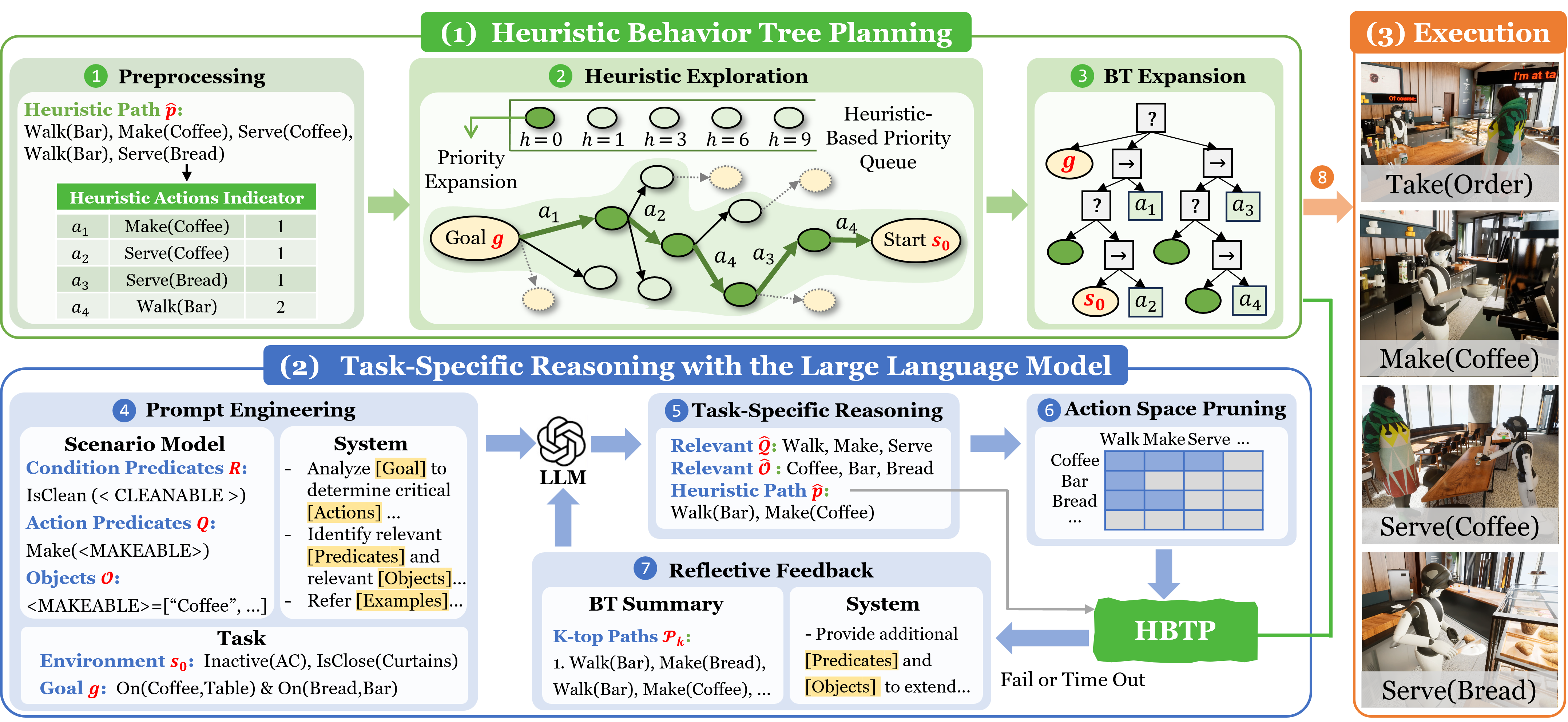}
\caption{An overview of our framework. (1) Before HBTP, a heuristic path reasoned by the LLM is used to construct a heuristic action indicator. 
During HBTP, condition nodes in the BT are ranked and expanded based on their heuristic values through exploration until the initial state is reached, with the BT expanding accordingly.
(2) For LLM reasoning, the scene and task are input to the LLM, which outputs task-relevant predicates and objects to prune the action space. If planning fails or times out, BT summaries are utilized to further refine the heuristics. (3) After HBTP, the produced BT is implemented into the robot to perform the task reactively and robustly.}
\label{fig:framework}
\end{figure*}



In recent years, pre-trained large language models (LLMs) have demonstrated strong task reasoning capabilities across multiple domains \cite{plaat2024reasoningllm,havrilla2024teachingllm,huang-chang-2023-towards,bai2024efficient,chen2024llama}. For instance, \cite{chen2024integrating} employs LLMs to generate task-relevant goals based on human instructions, facilitating efficient BT generation. SayCanPay \cite{hazra2024saycanpay} combines the world knowledge capabilities of LLMs with the systematicity of classical planning to generate practical and efficient plans. LLM-A* \cite{meng2024llm} uses LLM-generated waypoints to enhance pathfinding in large-scale environments.

Despite these advancements, LLMs have not yet been integrated into BT planning.
A major challenge is that LLMs may produce incorrect, redundant, or incomplete action sequences, which risks compromising the correctness and optimality of BTs. While LLMs can accelerate the BT planning process, it is equally critical to maintain the BT's safety, reactivity, and robustness.


To bridge this gap, this paper proposes Heuristic Behavior Tree Planning (HBTP), a reliable and efficient framework that integrates BT planning with LLM reasoning for BT generation. HBTP introduces a heuristic expansion process, offering two variants tailored for optimal and satisficing planning, along with a formal comparison of their efficiency and optimality. By leveraging the LLM for task-specific reasoning to generate a heuristic path, HBTP constructs a heuristic action indicator and implements a heuristic-driven BT expansion strategy, allowing BT planning to follow this path to expand efficiently.


While this heuristic expansion significantly accelerates planning with accurate LLM reasoning or a complete heuristic path, it is still slowed by missing actions. Therefore, we further propose action space pruning and reflective feedback to address the inaccuracies in LLM reasoning. The LLM is leveraged not only to reason about heuristic paths but also to generate task-relevant action predicates and objects, imposing dual constraints on the action space to reduce irrelevant searches and improve efficiency (see Figure \ref{fig:example}). Furthermore, automatic reflective feedback is employed, where the top-$k$ longest planned sequences are fed back to the LLM for refinement, enhancing its reasoning capabilities and improving the framework’s reliability.


Our main contributions are summarized as follows:
\begin{itemize}
\item We introduce HBTP, a reliable and efficient framework for BT generation that integrates BT planning with LLM reasoning. BT planning leverages the heuristic path generated by LLMs for efficient expansion.

\item We address LLM reasoning inaccuracies with action space pruning and reflective feedback to enhance both planning efficiency and reasoning accuracy.

\item We construct four BT planning datasets (\textit{RoboWaiter}\cite{chen2024integrating}, \textit{VirtualHome}\cite{puig2018virtualhome}, \textit{RobotHow-Small}, and \textit{RobotHow}\cite{liao2019synthesizing}) and validate HBTP's effectiveness, demonstrating its applicability to service robots across diverse environments.

\end{itemize}	

\section{Preliminaries}

A BT is a directed rooted tree structure in which the leaf nodes (including condition and action nodes) control the robot's perception and actions, while the internal nodes (including sequence and fallback nodes) handle the logical structuring of these leaves \cite{colledanchise2018behavior} (see Figure \ref{fig:example}). In the context of optimal BT planning \cite{chen2024integrating}, the problem is described as a tuple: \(<\mathcal{S},\mathcal{A}, \mathcal{M}, D, s_0,g>\), where \( \mathcal{S} \) is the finite set of environment states, \( \mathcal{A} \) is the finite set of actions, $\mathcal{M}$ is the action model, $D: \mathcal{A}\mapsto \mathbb{R}^+$ is the cost function,  $s_0$  is the initial state, $g$ is the goal condition. $<s_0,g>$ is the task. Both the state $s\in \mathcal{S}$ and the condition $c$ are represented by a set of literals $\mathcal{L}$, similar to STRIPS-style planning \cite{fikes1971strips}. If $c\subseteq s$, it is said that condition $c$ holds in the state $s$. 

The state transition affected by action $a\in \mathcal{A}$ can be defined as a triplet \( \mathcal{M}(a)=<pre(a),add(a),del(a)> \), comprising the precondition, add effects, and delete effects of the action. After the action's execution, the subsequent state $s'$ will be: $s' = s \cup add(a) \setminus del(a)$. For each BT \( \mathcal{T} \) capable of transitioning $s_0$ to $g$ in a finite time, if the actions required to transition from \( s_0 \) to \( g \) are \( a_1, a_2, ..., a_n \), the optimal BT planning is to find the optimal BT \( \mathcal{T}_* \):
\begin{equation}
\mathcal{T}_* = \mathop{\arg \min}_{\mathcal{T}\in \mathscr{T}} D(\mathcal{T}) =\mathop{\arg \min}_{\mathcal{T}\in\mathscr{T}} \sum_{i=1}^{n} D(a_i)
\end{equation}

In the predicate logic representation, actions and literals are represented using a tuple $<\mathcal{R},\mathcal{Q},\mathcal{O}>$
, where $\mathcal{R}$ is the condition predicate set, $\mathcal{Q}$ is the action predicate set, and $\mathcal{O}$ is the object set. For a predicate $q \in \mathcal{R} \cup \mathcal{Q}$, its domain may be a subset of the Cartesian product of $\mathcal{O}$. In this formulation, a literal can be denoted as $l = r(o_1, ..., o_i), r\in \mathcal{R}$ while an action can be denoted as $a = q(o_1, ..., o_j), q\in \mathcal{Q}$. 

OBTEA \cite{chen2024integrating} is a complete and sound algorithm to find the optimal and robust BT given the goal condition represented in Disjunctive Normal Form (DNF). In this paper, we focus on the goals represented as conjunctions of literals such as \constant{IsClean(Carrot)}$\wedge$\constant{IsIn(Carrot,Fridge)} for simplicity. To deal with the disjunction of goals, one can simply connect sub-goal trees using a fallback node. The complexity of OBTEA is $O(|\mathcal{A}||\mathcal{S}|\log(|\mathcal{S}|))$, where the size of the environment states $|\mathcal{S}|$ can reach $O(|\mathcal{A}|!)$ in the worst case. Massive computational consumption leads to a strong need for more efficient BT planning algorithms to make it practical for real-world applications.

\section{Methodology}
In this section, we first design heuristics and provide a formal analysis of the two heuristic variants (\ref{subsec:oracle_heuristic}). We then present the HBTP algorithm's workflow (\ref{subsec:method_HBTP}), followed by the practical implementation of task-specific reasoning with LLMs, including input/output structures, action space pruning, and reflective feedback (\ref{subsec:method_llm}). An overview of our framework is shown in Figure~\ref{fig:framework}.


\subsection{Heuristics Design and Formal Analysis}
\label{subsec:oracle_heuristic}

\paragraph{Paths in BT Planning} We begin with the definition of paths in BT planning. For a BT found by OBTEA or other BT planning algorithms \cite{cai2021bt}, there usually exist many paths to achieve the goal condition. Expanding all candidate paths explored during search is the main strategy for planned BTs to ensure robustness in execution. We define the set of all finite paths from $s$ to $g$ as $\mathcal{P}(s,g)$. For an optimal BT planning problem, the optimal path $p^* \in \mathcal{P}(s_0,g)$ is an action sequence $p^* = (a_1, a_2, \dots, a_n)$ that minimizes total cost, $D(p^*) \leq D(p), \forall p \in \mathcal{P}(s_0,g)$. In this paper, by \textit{optimal path $p^*$}, we refer to any of the optimal paths from $s_0$ to $g$. If there are multiple optimal paths, any of them can be utilized for task-specific heuristic.

\paragraph{Heuristic Action Indicator} The core idea of leveraging task-specific heuristics with LLMs is to utilize the predicted optimal path as the heuristic path to guide the planning procedure. The algorithm prioritizes actions within this heuristic path. To formalize this, we define the set of all actions in the heuristic path as $\mathcal{A}(p^*)$ and the action indicator function $I: \mathcal{P} \times \mathcal{A} \mapsto \mathbb{N}$. The function $I(p, a)$ denotes the frequency of action $a$ in path $p$, remaining unaffected by any errors in action order. Since OBTEA prioritizes low-cost actions, we reduce the cost of actions in the heuristic path $a \in \mathcal{A}(p^*)$ as a heuristic during planning, provided that for the current search path $p$, $I(p, a) \leq I(p^*, a)$. 

\paragraph{Optimal Heuristics} Intuitively, we can divide the heuristic cost of prioritized actions by a large number $\alpha \gg 1$. If the heuristics path is $\hat{p}$, the total heuristic cost of path $p$ can be defined as:
\begin{align}
h^\alpha(p,\hat{p}) \dot{=} & \sum_{a \in \mathcal{A}(p)} \max \left[ 0, I(p,a) - I(\hat{p},a) \right] D(a) \nonumber \\
& + \frac{1}{\alpha} \sum_{a \in \mathcal{A}(p)} \max \left[ 0, I(\hat{p},a) - I(p,a) \right] D(a)
\end{align}

If the planning algorithm can find the path $p=\mathop{\arg \min}_{p\in\mathcal{P}(s_0,g)} h^\alpha(p,\hat{p})$ and the heuristics path $\hat{p}$ has no non-optimal actions, then this heuristic is optimal.
\begin{proposition} \label{pro:optimality}
Given a task $<s_0,g>$ and a heuristic path $\hat{p}$, if $I(\hat{p},a)\leq I(p^*,a), \forall a\in\mathcal{A}$ and $\alpha \geq 1$, we have $\forall p \in \mathcal{P}(s_0,g), h^\alpha (p^*,\hat{p}) \leq h^\alpha (p,\hat{p})$.
\end{proposition}
\begin{proof}(sketch) For any $p\in \mathcal{P}(s_0,g)$, we can always rename finite actions (keep the same action model and cost) in $p$ to let $I(p,a)\leq 
I(\hat{p},a),\forall a\in \mathcal{A}(\hat{p})$. Without losing generality, we have:
\begin{equation}
    \begin{aligned}
        &h^\alpha(p^*, \hat{p}) - h^\alpha(p, \hat{p}) = D(p^*) - D(p) \\
        &+ \left( \frac{1}{\alpha} - 1 \right) \sum_{a \in \mathcal{A}(\hat{p})} \left[ I(p^*, a) - I(p, a) \right] D(a) \leq 0\\
    \end{aligned} 
    \label{eq:1}
\end{equation}
\end{proof} 
That proves the proposition \ref{pro:optimality}. In the ideal case, we hope to not search for any other action until the search for the heuristic path is finished. To achieve this, $\alpha$ should satisfy $\frac{D(p^*)}{\alpha} < D(a), \forall a\in \mathcal{A}$, which means $\alpha > \frac{D(p^*)}{\min_{a\in \mathcal{A}} D(a)}$. This is a theoretical lower bound on $\alpha$, but in practice we just need to set $\alpha$ to a very large value. We acknowledge the strong assumption of no suboptimal actions in $\hat{p}$, but this heuristic, using $\frac{D(p^*)}{\alpha}$ and retaining original action costs, enables near-optimal BT planning under $\hat{p}$.









\paragraph{Satisficing Heuristics} When \( a \to \infty \), we can get the heuristic function for satisficing planning:
\begin{equation}
h^\infty(p,\hat{p}) = \sum_{a \in \mathcal{A}(p)} \max \left[ 0, I(p,a) - I(p^*,a) \right] D(a) 
\end{equation}
In this case, we treat all the cost of actions included in the heuristic path as $D(a)=0$. Although the relaxation of the optimal ordering of actions may somewhat affect the optimality of BT planning, this heuristic can achieve very fast planning speeds in practical applications.
\begin{proposition} \label{pro:optimality2}
Given a task $<s_0,g>$ and a heuristic path $\hat{p}$, if $I(\hat{p},a) = I(p^*,a), \forall a\in\mathcal{A}$, we have $\forall s \in \mathcal{S}, \forall p \in \mathcal{P}(s,g), h^\infty (p^*,\hat{p}) = 0 \leq h^\infty (p,\hat{p})$.
\end{proposition}
It means the optimal path $p^*=(a_1,a_2,...,a_n)$ can be identified in only $n$ exploration steps, reducing computational complexity to $O(|\mathcal{A}(p^*)|)$ in the best case. The experiments show that the satisficing heuristic has the lowest planning time with a negligible cost increment of produced BTs.

\subsection{Heuristic Behavior Tree Planning} 
\label{subsec:method_HBTP}



The pseudocode of HBTP is detailed in Algorithm \ref{alg:hobtea}, introducing a heuristic-driven BT expansion strategy based on OBTEA \cite{chen2024integrating}. We predefine the action space $\mathcal{A}$, including the implementation of each action and their action models (\(pre, add, del\)). Action implementations can be learned through methods like reinforcement learning \cite{10801816,Bai_Zhang_Tao_Wu_Wang_Xu_2023,cai2023task2morph}, and action models can be acquired using various existing techniques \cite{arora2018review}. Like OBTEA, HBTP uses backward planning starting from the goal \( g \). Exploration traverses all actions that satisfy condition \( c \) without violating its effects, computes a new condition \( c_{a} \), and adds it to the exploration queue. Expansion incorporates these actions and new conditions into the BT.
\hypersetup{pdfborder={0 0 0}} The key modifications, highlighted in blue, employ the dynamic and path-specific heuristic cost $h(c)$ to determine the expansion priorities of condition node $c$, replacing the reliance on the fixed action cost $D(c)$. We also track the heuristic action indicator $\forall a\in \mathcal{A}, I(c,a)$ for each $c$ (lines \ref{line:initTimes}, \ref{line:updateI}-\ref{line:decreaseI}), which counts how many prior actions are absent in the current path reaching $c$. If $I(c,a)>0$, the heuristic cost is calculated as $h(a)=D(a)/\alpha$ (with $h(a)=0$ in the satisficing heuristic); otherwise, $h(a)=D(a)$ (line \ref{line:heuristicA}).
\hypersetup{pdfborder={0 0 1}}
\begin{proposition} \label{pro:non-optimality}
Given a task $<s_0,g>$ and a heuristic path $\hat{p}$, even if $I(\hat{p},a) = I(p^*,a), \forall a\in\mathcal{A}$, HBTP using satisficing heuristic may produce a BT with an sub-optimal path $p'\in \mathcal{P}(s,g)$, $\exists p\in \mathcal{P}(s,g), h^\infty (p,\hat{p})\leq h^\infty (p',\hat{p})$.
\end{proposition}
\begin{proof} The sub-optimality is due to the pruning of expanded conditions (line \ref{line:ifExplored}). For example, if $p_1,p_2$ leads to $c_1,c_2$ respectively, and $h^\infty (p_1,\hat{p}) = h^\infty (p_2,\hat{p})=0$, but only $p_1$ can finally lead to $p^*$. In this case, if $c_2$ is expanded first and $c_1\supseteq c_2$, the optimal path will be pruned.
\end{proof}
Despite the sub-optimal nature of the satisficing heuristic, we recommend its use over the optimal heuristic for three main reasons: (1) Based on the experimental findings in Section \ref{sec:bt-results}, the additional cost of the satisficing heuristic is minimal; (2) Inaccuracies in LLM reasoning can also result in sub-optimal BT even when the optimal heuristic is employed; (3) The satisficing heuristic can be much more efficient than the optimal one due to the expanded condition pruning, particularly with longer optimal path lengths.

\definecolor{dgreen}{rgb}{0,0.525,0}
\definecolor{blue}{rgb}{0,0,1}
\renewcommand{\algorithmicrequire}{\textbf{Input:}}
\renewcommand{\algorithmicensure}{\textbf{Output:}}   
\renewcommand{\algorithmiccomment}[1]{\hfill\textcolor{dgreen}{\(\triangleright\) #1}}

\renewcommand{\thealgorithm}{\arabic{algorithm}}  

\begin{algorithm}[t]
\small
\caption{HBTP}
\label{alg:hobtea}
\begin{algorithmic}[1]
    \linespread{0.95}\selectfont 
    \REQUIRE BT Planning Problem: $<\mathcal{A},\mathcal{M},D,s_0,g>$, predicted heuristic path  $\hat{p}$
    \ENSURE Final Expansion BT \( \mathcal{T}_* \)
    \StateCommentLabel{$\mathcal{T} \gets Fallback(g)$}{\algorithmiccomment{init BT}}{line:initBT}
    \StateCommentLabel{$h(g)\gets 0$}{\algorithmiccomment{for $\forall c, c\neq g, h(c) \gets +\infty$}}{line:initCost}
    \StateCommentLabel{\textcolor{blue}{$I(g,a)\gets I(p^*,a), \forall a\in\mathcal{A}$}}{\algorithmiccomment{init action indicator for $g$}}{line:initTimes}
    \StateCommentLabel{$\mathcal{C}_U\gets \{g\},\mathcal{C}_E\gets \emptyset$}{\algorithmiccomment{explored and expanded conditions}}{line:cu}
    \WHILE{$\mathcal{C}_U \neq \emptyset$}
    \StateCommentLabel{$c\gets \arg \min_{c\in\mathcal{C}_U} (h(c))$}{\algorithmiccomment{explore and expand $c$}}{line:pickC}
    \FOR{\textbf{each} $a\in A$} \label{line:fora}
    \IF{$(c\cap (pre(a)\cup add(a)\setminus del(a)) \neq \emptyset)$ and $(c\setminus del(a)=c)$} \label{line:forExploration}
    \StateCommentLabel{$c_{a}\gets pre(a) \cup c \setminus add(a)$}{\algorithmiccomment{exploration}}{line:exploration}
    \StateCommentLabel{\textcolor{blue}{$h(a)\gets D(a)/\alpha $ \textbf{if} $I(c,a)>0$ \textbf{else} $D(a)$ }}{\algorithmiccomment{heuristic cost of $a$}}{line:heuristicA}
    \IF{\textcolor{blue}{$c_a \not\supseteq c', \forall c'\in \mathcal{C}_E$ and $h(c) + h(a) < h(c_{a})$}} \label{line:ifExplored}
    \StateCommentLabel{\textcolor{blue}{$h(c_{a})\gets h(c) + h(a)$}}{\algorithmiccomment{update $h(c_a)$}}{line:updateH}
    \StateCommentLabel{\textcolor{blue}{$I(c_{a},a')\gets I(c,a'), \forall a'\in \mathcal{A}$ }}{\algorithmiccomment{update $I(c_{a},\cdot)$}}{line:updateI}
    \StateCommentLabel{\textcolor{blue}{$I(c_{a},a)\gets I(c_{a},a)-1$}}{\algorithmiccomment{decrease $I(c_{a},a)$}}{line:decreaseI}
    \STATE $\mathcal{C}_U\gets \mathcal{C}_U\cup \{c_{a}\}$ \algorithmiccomment{add $c_a$ to $\mathcal{C}_U$}
    \STATE $\mathcal{M}(c_{a})\gets Sequence(c_a,a)$
    \ENDIF
    \ENDIF
    \ENDFOR
    \StateCommentLabel{$\mathcal{C}_U\gets \mathcal{C}_U\setminus \{c\}, \mathcal{C}_E\gets \mathcal{C}_E\cup \{c\}$}{\algorithmiccomment{move $c$ from  $\mathcal{C}_U$ to $\mathcal{C}_E$}}{line:removeC}
    \IF{$c\neq g$}
    \StateCommentLabel{$\mathcal{T} \gets Fallback(\mathcal{T},\mathcal{M}(c))$}{\algorithmiccomment{expansion}}{line:expand}
    \IF{$c \subseteq s_0$}
    \RETURN $\mathcal{T}_*$ \algorithmiccomment{return final BT}
    \ENDIF
    \ENDIF
    \ENDWHILE
\end{algorithmic}
\end{algorithm}

\subsection{Task-Specific Reasoning with Large Language Models} 
\label{subsec:method_llm}


Previous sections showed that heuristic BT planning is efficient with complete or slightly redundant reasoning from LLMs. However, efficiency may deteriorate when certain actions are omitted. This section explores mitigating LLM inaccuracies via action space pruning and reflective feedback to enhance both reasoning accuracy and planning efficiency.

\paragraph{Prompt Engineering} The key to using pre-trained LLMs to carry out task-specific reasoning for BT planning without fine-tuning is prompt engineering \cite{singh2022progprompt,ahn2022can}. Our prompts consist of three components: (1) Scenario. The scenario is defined as a three-tuple $<\mathcal{R},\mathcal{Q},\mathcal{O}>$, as shown in Figure \ref{fig:framework}. The condition predicate set $\mathcal{R}$ is used to indicate the goal understanding. The action predicate set $\mathcal{Q}$ and the object set $\mathcal{O}$ are used to reason the planning procedure. The objects are categorized for a compact representation. (2) 
Few-shot demonstrations. These provide task-specific examples within the prompt to guide the LLM in reasoning without fine-tuning. (3) System. An explanatory text provides context and the prompt for instructing the LLM to produce task-specific reasoning in the correct format. 


\paragraph{Task-Specific Reasoning} In a certain scenario $<\mathcal{R},\mathcal{Q},\mathcal{O}>$, given a task $<s_0,g>$, we want the LLM to reason the planning heuristics: $\hat{\mathcal{Q}},\hat{\mathcal{O}},\hat{p} = LLM(s_0,g)$. $\hat{\mathcal{Q}}$ is the relevant action predicates, $\hat{\mathcal{O}}$ is the relevant objects, and $\hat{p}$ is the heuristic path. The outputs will first be used to prune the action space (as detailed below), followed by heuristic search using  $\hat{p}$ within the pruned space (as detailed in the preceding section). To improve the grammatical accuracy of LLM outputs, we apply reflective feedback similar to \cite{chen2024integrating}. We use a grammar checker to find the syntax and semantic errors, and add them to a blacklist, allowing the LLM to output again and prevent making the same mistakes again.






\paragraph{Action Space Purning} For each task, pruning irrelevant actions from the search space accelerates planning by reducing unnecessary searches and focusing on relevant ones. The action space, defined as the Cartesian product of the action predicate set and the object set, can be pruned by separately reducing the action predicate set and the object set. The action predicate set \( \mathcal{Q}(\hat{p}) \)  and  object set \( \mathcal{O}(\hat{p}) \)  of the heuristic path \( \hat{p} = (a_1, a_2, \dots, a_n) \) are defined as the union of predicates  and objects involved in each action along the path: $\quad \mathcal{Q}(\hat{p}) = \bigcup_{i=1}^n \mathcal{Q}(a_i)  $ and $  \mathcal{O}(\hat{p}) = \bigcup_{i=1}^n \mathcal{O}(a_i) $. Consequently, the pruned action predicate set \( \mathcal{Q}^- \) and object set \( \mathcal{O}^- \) are given by \( \mathcal{Q}^- = \hat{\mathcal{Q}} \cup \mathcal{Q}(\hat{p}) \) and \( \mathcal{O}^- = \hat{\mathcal{O}} \cup \mathcal{O}(\hat{p}) \), respectively. Utilizing additional relevant action predicates and objects to prune the action space, rather than relying solely on those from the heuristic path, mitigates the issue of missing actions. The Cartesian product of the two is defined as $\mathcal{Q}^- \times \mathcal{O}^- = \{ q(o_1,o_2,...) | q \in \mathcal{Q}^-, o_1,o_2,... \in \mathcal{O}^-$. All valid actions in this set form the final pruned action space \( \mathcal{A}^- = (\ \mathcal{Q}^-  \times \mathcal{O}^- ) \cap \mathcal{A} \), where \( \mathcal{A} \) is the full valid action set.



\paragraph{Reflective Feedback} Action space pruning must balance efficiency and completeness to prevent solution failures. Our strategy tolerates action redundancy and triggers automated feedback upon planning failure or timeout. The feedback to LLM consists of two parts: (1) a summarized BT represented by action sequences of the top-\(k\) longest paths \(\mathcal{P}_k = \{ p_1, \dots, p_k \}\), where each \(p_i = \langle a^{(i)}_1, \dots, a^{(i)}_{n_i} \rangle\). It optimizes token efficiency while retaining critical information, exposing decision bottlenecks (e.g., planning halts due to missing actions). (2) the action predicate set and object set not included in the existing pruned action space, denoted as \(\mathcal{Q}' = \mathcal{Q} \setminus \mathcal{Q}^-\) and \(\mathcal{O}' = \mathcal{O} \setminus \mathcal{O}^-\), enabling LLM to expand the pruned space by reintroducing essential elements. Based on this feedback, LLM regenerate task-relevant predicates, objects, and heuristic paths, iteratively expanding the action space and enabling further BT planning, thereby improving the likelihood of success.

\section{Experiments}

Our experiments consists of two parts: (1) Analysis of HBTP's theoretical bounds, focusing on two heuristic methods—HBTP-O (optimal heuristics) and HBTP-S (satisficing heuristics)—as well as action space pruning. (2) Evaluation of LLM reasoning with reflective feedback across four everyday service scenarios. Our code is available at \href{https://github.com/DIDS-EI/LLM-HBTP}{https://github.com/DIDS-EI/LLM-HBTP}.





\subsection{Settings}


\paragraph{Datasets and Scenarios} We constructed four BT planning datasets, each containing 100 data, based on the following four scenarios respectively: \textit{RoboWaiter} \cite{chen2024integrating} 
, \textit{VirtualHome} \cite{puig2018virtualhome}
, \textit{RobtoHow} \cite{liao2019synthesizing} and \textit{RobotHow-Small} (subset of \textit{RobotHow} with fewer objects).
The problem scales of the scenarios are shown in Table \ref{tab:4_scenarios_results}.


\paragraph{Baselines} We main compare our approaches with two baselines: (1) BT Expansion \cite{cai2021bt}: A sound and complete algorithm for BT planning, where the optimality is not guaranteed. (2) OBTEA \cite{chen2024integrating}: An improved algorithm based on BT Expansion for optimal BT planning.


\begin{figure}[h]
\centering
\begin{subfigure}[t]{0.48\linewidth}
    \centering
    \includegraphics[width=\linewidth]{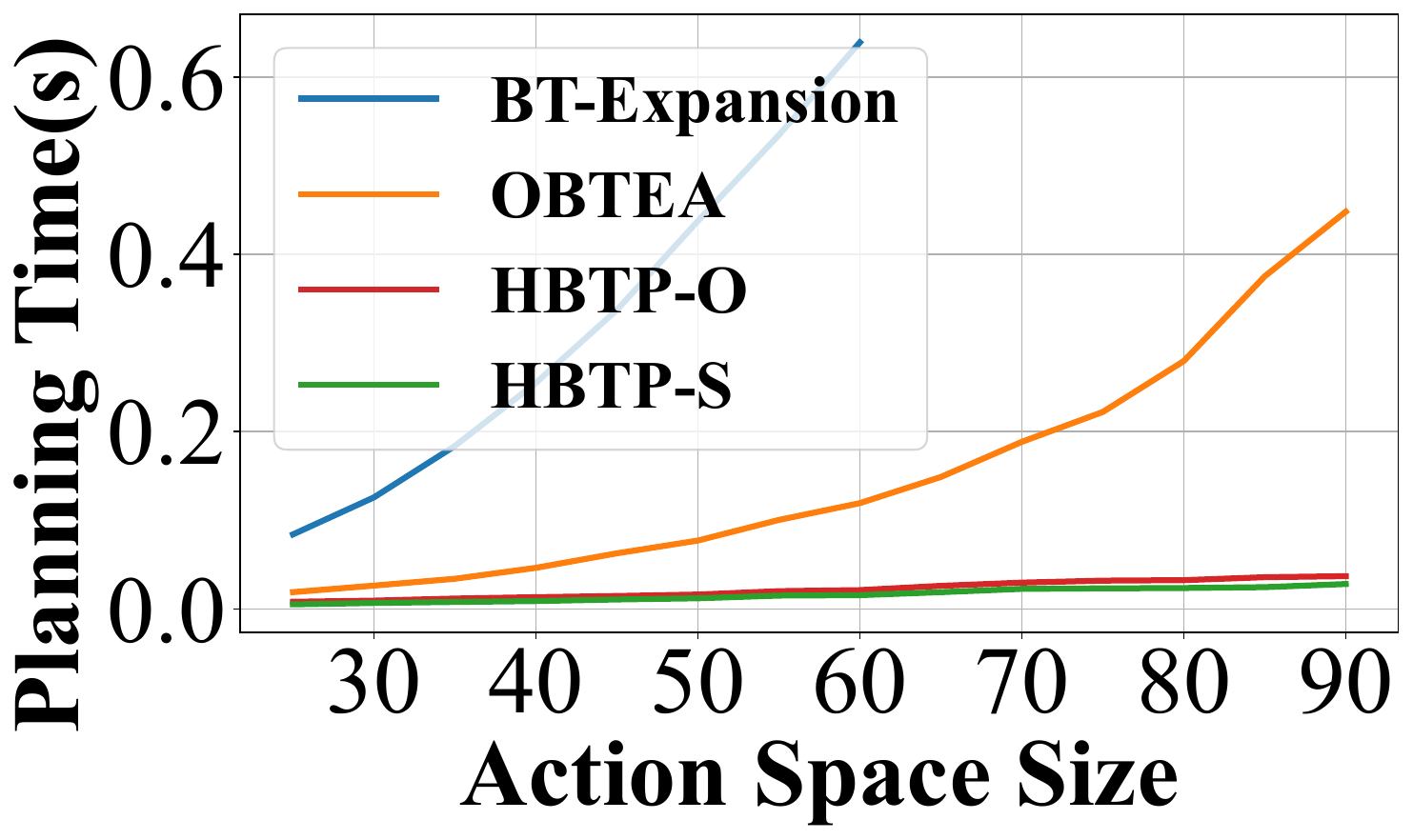}
    \caption{}
    \label{fig:EXP_BT_1_time_vs_size}
\end{subfigure}
\hfill
\begin{subfigure}[t]{0.48\linewidth}
    \centering
    \includegraphics[width=\linewidth]{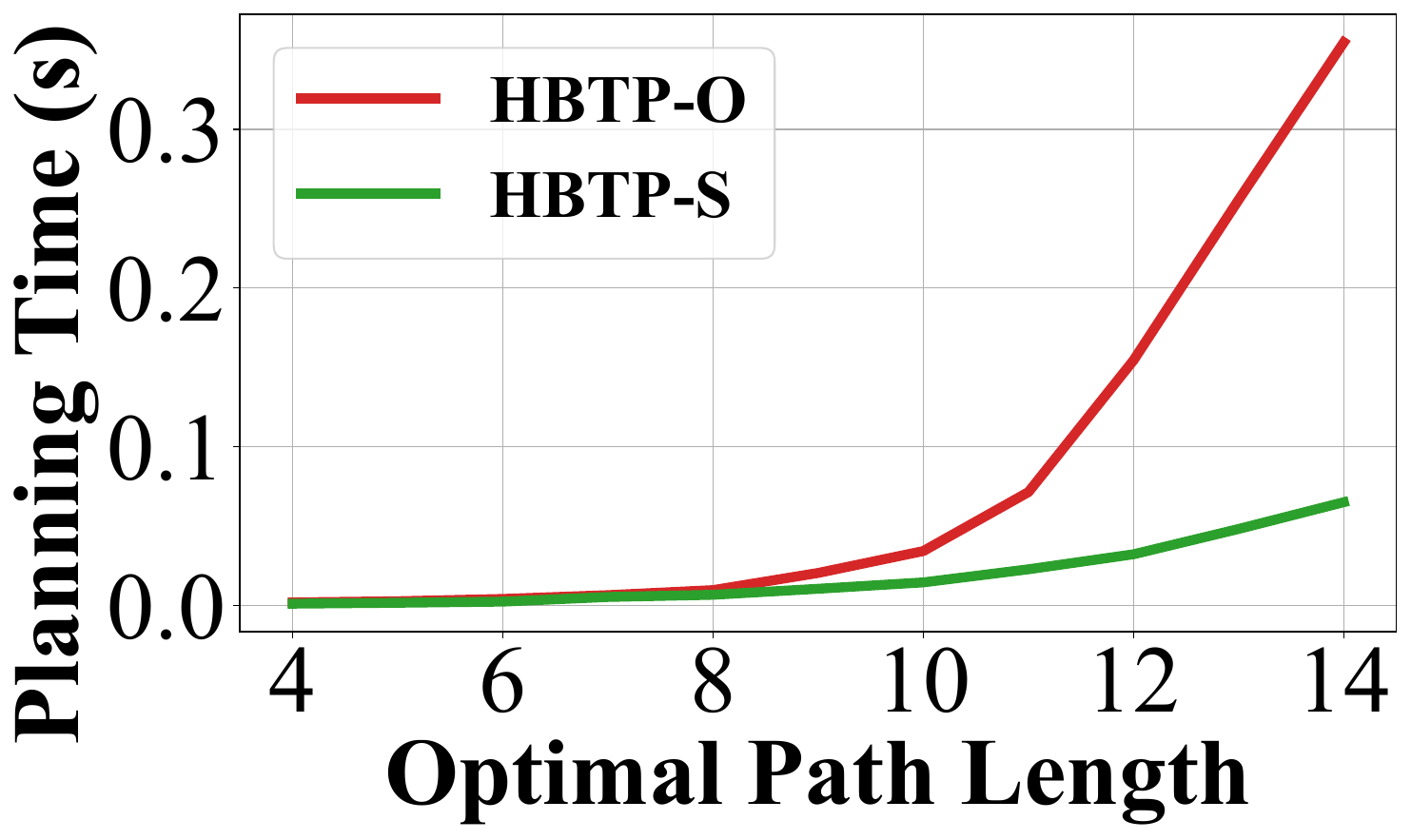}
    \caption{}
    \label{fig:EXP_BT_1_time_vs_act_len}
\end{subfigure}
\caption{Comparison of planning time.} 
\label{fig:EXP_BT_1_time_vs_size_combined}
\end{figure}

\subsection{Theoretical Bounds of HBTP and Action Space Pruning}
\label{sec:bt-results}

\paragraph{Metrics} The metrics used in this section include: (1) Planning time. The total time spent for planning the BT is related to the computing performance of the platform. Our experiments are all performed in the AMD Ryzen 9 5900X 12-Core Processor (3.70 GHz). (2) The number of explored conditions ($|\mathcal{S}^-|$). The explored conditions $\mathcal{S}^-$ during planning is a subset of full state space $\mathcal{S}^-\subseteq \mathcal{S}$. It can also reflect the planning speed but is platform-independent. (3) Total cost of planned BT ($D(\mathcal{T})$). It reflects the optimality of BT planning algorithms. All experiments in this section is performed in the \textit{RobotHow-Small} dataset.

\paragraph{Evaluation of HBTP} As shown in  Figure~\ref{fig:EXP_BT_1_time_vs_size}, the planning time of two baseline algorithms increase rapidly as the action space size $|\mathcal{A}|$ grows. In contrast, the planning time of HBTP-O and HBTP-S both shows a slow increase. Additionally, as depicted in Figure~\ref{fig:EXP_BT_1_time_vs_act_len}, as the length of the optimal path grows, the optimal heuristic spent much more time for planning the optimal BT with those predicted optimal actions. While the HBTP-S maintains lower planning times because of the expended condition pruning strategy.

\paragraph{Evaluation of Action Space Pruning} As shown in Table~\ref{tab:performance_metrics}, in the full action space, all four algorithms may experience planning timeout (>5s). In contrast, in the pruned action space, planning times are notably decreased. Specifically, for HBTP-O and HBTP-S, even though the number of expanded conditions $|\mathcal{S}^-|$ does not decrease significantly, the saved planning time is substantial due to the reduction in actions search for each condition exploration. Another experimental phenomenon is that for the two non-optimal algorithms, BT Expansion and HBTP-S, the total cost of BT planned by them is very close to the optimal cost. This suggests to us that in household service planning tasks, there are few non-optimal paths, so we believe the HBTP-S is more advantageous in these scenarios.



\paragraph{Error Tolerance} We then analyze the error tolerance level of the heuristics. Assuming the optimal path is $p^*$, we randomly decrease the actions in it to reduce the prediction accuracy, and randomly add non-optimal actions to increase the error rate. As shown in Figure~\ref{fig:h0_h1}, as correct rate increases, the expanded conditions $|\mathcal{S}^-|$ decreases for both heuristics. For the optimal heuristic, $|\mathcal{S}^-|$ significantly increases especially with the higher correct (\ref{fig:plot_h1_EN}). For the HBTP-S, the $|\mathcal{S}^-|$ is much more stable when the error rate get higher (\ref{fig:plot_h0_EN}), as well as the $D(\mathcal{T})$ (\ref{fig:plot_h0_cost}). These results again demonstrate the superior of the satisficing heuristic over the optimal one in terms of error tolerance.

\begin{figure*}[t]
    \centering
    \begin{subfigure}[b]{0.31\textwidth}
        \centering
        \includegraphics[width=\textwidth]{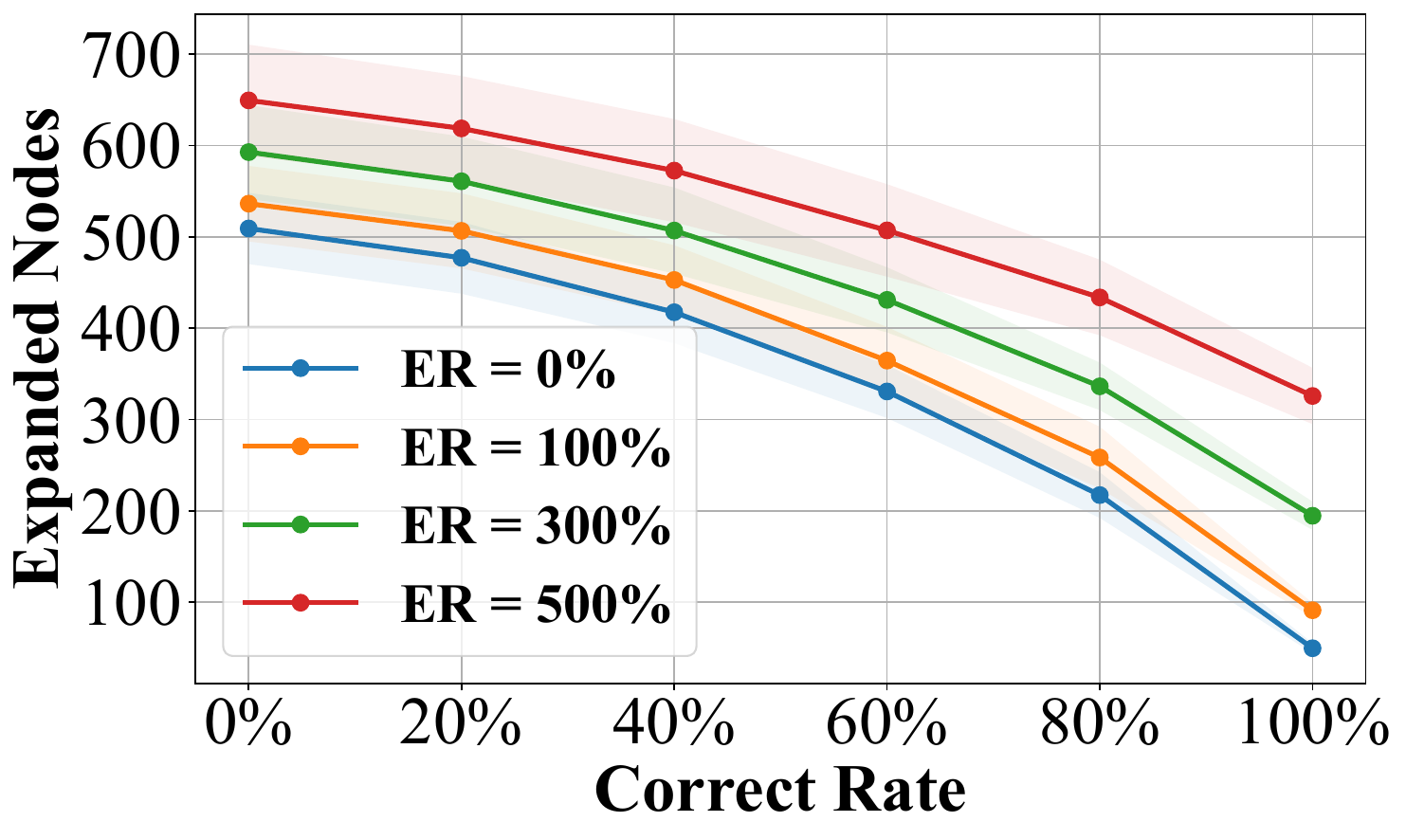}
        \caption{$|\mathcal{S}^-|$ of HBTP-O}
        \label{fig:plot_h1_EN}
    \end{subfigure}
    \hfill
    \begin{subfigure}[b]{0.31\textwidth}
        \centering
        \includegraphics[width=\textwidth]{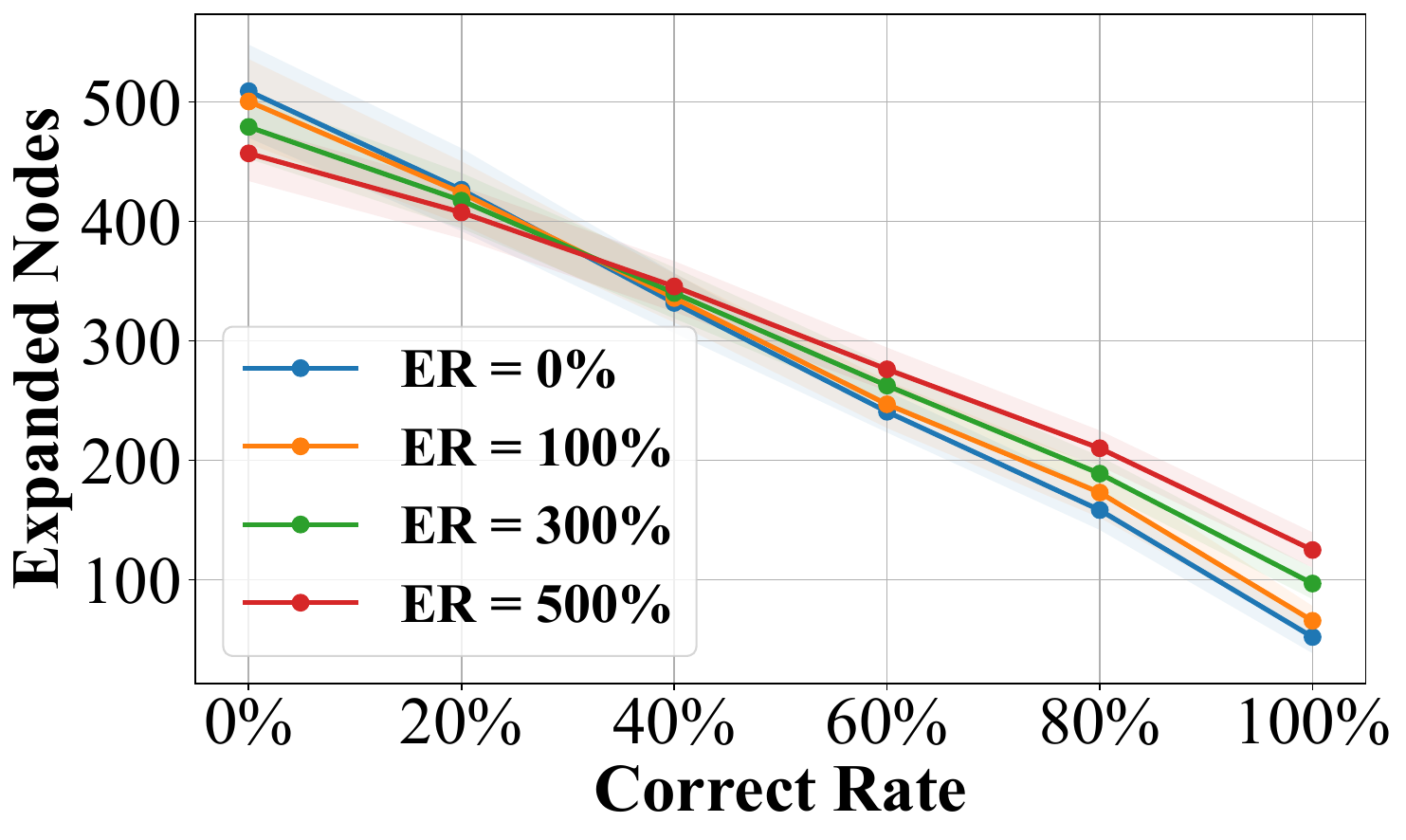}
        \caption{$|\mathcal{S}^-|$ of HBTP-S}
        \label{fig:plot_h0_EN}
    \end{subfigure}
    \hfill
    \begin{subfigure}[b]{0.31\textwidth}
        \centering
        \includegraphics[width=\textwidth]{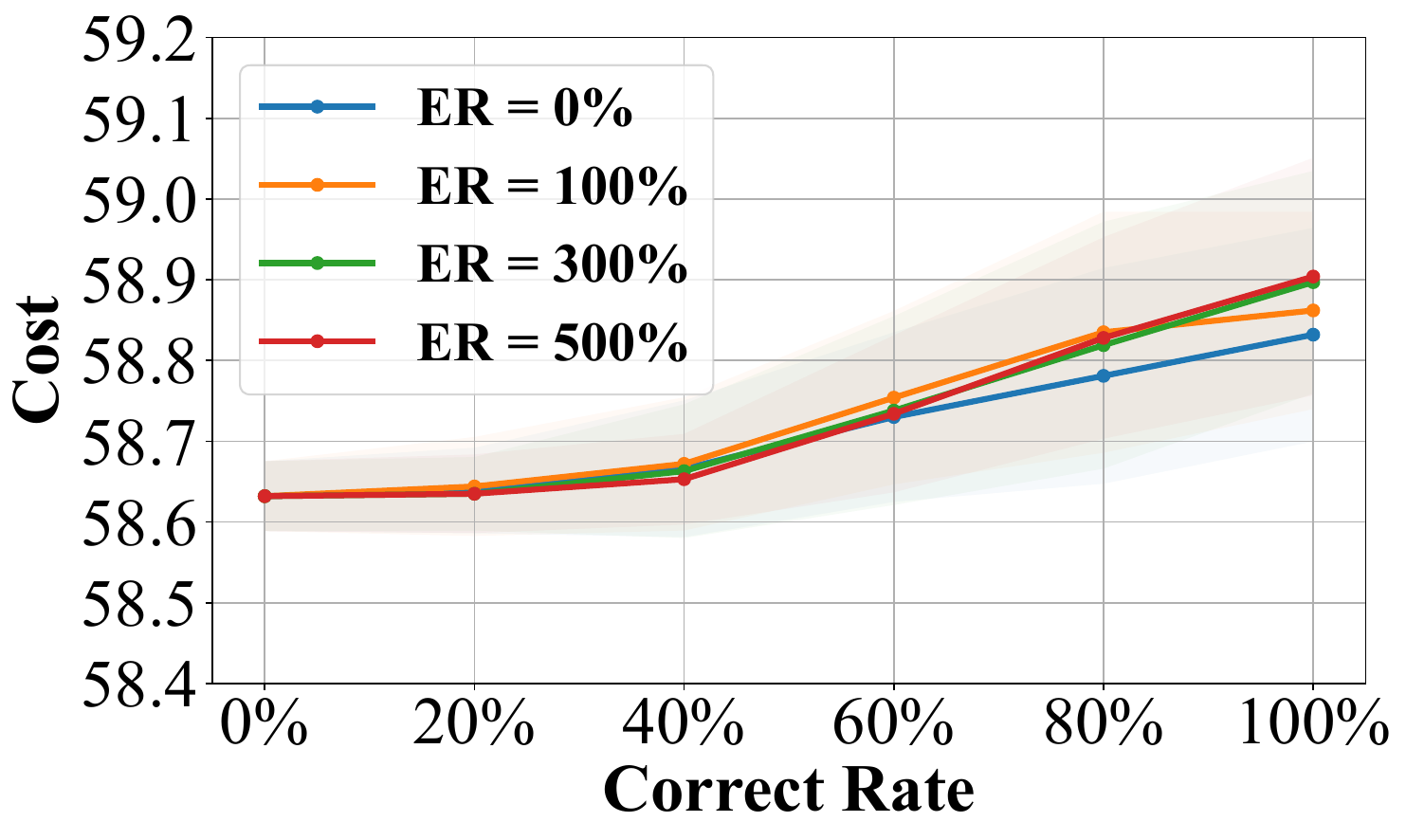}
        \caption{$D(\mathcal{T})$ of HBTP-S}
        \label{fig:plot_h0_cost}
    \end{subfigure}
    \caption{Impact of correct and error rate for two HBTP algorithms with different heuristics. ER stands for error rates.} 
    \label{fig:h0_h1}
\end{figure*}

\begin{table}[t]
\centering
\setlength{\tabcolsep}{2.5pt} 
\caption{Impact of action space pruning on BT planning}
\label{tab:performance_metrics}
\small
    \begin{tabular}{cccccc}
        \toprule
        \small
        \textbf{\makecell{Action \\ Space}} & \textbf{Heuristic} & \textbf{$|\mathcal{S}^-|$} & \textbf{\makecell{Timeout\\ Rate}} & \textbf{\makecell{Planning\\ Time(s)}} & \textbf{$D(\mathcal{T})$} \\
        \midrule
        \multirow{4}{*}{\textbf{Full}} & \textbf{BT Expansion} & 117.8 & 79\% & 4.151 & - \\
        & \textbf{OBTEA} & 405.3 & 82\% & 4.126 & - \\
        & \textbf{HBTP-O} & 27.71 & 9\% & 1.243 & - \\
        & \textbf{HBTP-S} & \textbf{23.87} & \textbf{4\%} & \textbf{0.887} & - \\
        \midrule
        \multirow{4}{*}{\textbf{Pruned }} & \textbf{BT Expansion} & 41.0 & 0\% & 0.016 & 59.86 \\
        & \textbf{OBTEA} & 34.5 & 0\% & 0.006 & \textbf{58.64} \\
        & \textbf{HBTP-O} & 25.43 & 0\% & 0.004 & \textbf{58.64} \\
        & \textbf{HBTP-S} & \textbf{18.17} & 0\% & \textbf{0.003} & 58.83 \\
        \bottomrule
    \end{tabular}
\end{table}

\begin{table*}[t]
    \centering
    \setlength{\tabcolsep}{4.5pt} 
    \caption{Performance of LLM reasoning with HBTP in four scenarios (50 tasks per scenario; averaged over three runs)}
    \label{tab:4_scenarios_results}
    \small
    \begin{tabular}{lccccccccccccc}
        \toprule
        \multirow{3}{*}{\textbf{Scenarios}} & \multicolumn{3}{c}{\multirow{2}{*}{\textbf{Problem Scale}}} & \multicolumn{2}{c}{\textbf{Oracle}} & \multicolumn{8}{c}{\multirow{2}{*}{\textbf{LLM Reasoning Performance}}} \\ 
        & & & & \multicolumn{2}{c}{\textbf{Heuristics}} & & & & & & & \\
        \cmidrule(lr){2-4} \cmidrule(lr){5-6} \cmidrule(lr){7-14}
        & $|\mathcal{A}|$ & \textbf{$|\mathcal{Q}|$} & \textbf{$|\mathcal{O}|$} & $|\mathcal{A}^*|$ & $|p^*|$ & \textbf{NF-SR(\%)} & \textbf{1F-SR(\%)}& \textbf{3F-SR(\%)} & $|\mathcal{A}^-|$ & $|\hat{p}|$  &  $|\mathcal{S}^-|$ & \textbf{\textbf{Time(s)}} & \textbf{B-Time(s)} \\
        \midrule
        \textbf{RoboWaiter} & 174 & 6 & 31 & 2.46 & 2.46 & 92 & 98.67 & 99.33 & 2.54 & 2.46 & 3.99 & \textbf{0.0004} & 0.0008 \\
        \textbf{VirtualHome} & 480 & 11 & 36 & 8.02 & 6.58 & 90.67 & 97.33  & 98 & 7.93 & 6.66 & 28.39 & \textbf{0.0053} & >1m \\
        \textbf{RobotHow-S} & 155 & 16 & 17 & 5.52 & 4.54 & 74 & 80.67 & 84.67 & 8.63 & 4.55 & 13.98 & \textbf{0.0041} & 0.5747 \\
        \textbf{RobotHow} & 7338 & 16 & 126 & 12.3 & 8.04 & 52.67 & 77.33 & 83.33 & 13.97 & 8.19 & 103.16 & \textbf{0.0344} & >1h \\
        \bottomrule
    \end{tabular}
\end{table*}

\subsection{LLM Performance with Reflective Feedback} \label{sec:LLM-results}


\paragraph{Metrics} (1) Success Rate (SR). Proportion of successful BT planning under conditions of no feedback (NF), up to one feedback (1F), and up to three feedbacks (3F). In the feedback scenarios, the top 3 longest paths are selected as BT summaries ($k=3$).
(2) Pruned action space size ($|\mathcal{A}^-|$). It is the final size of the action space when BT planning produces a feasible solution. The action space $\mathcal{A}^-$ may expand based on feedback if the LLM's reasoning is inadequate. All experiments use \textit{gpt-4o-mini-2024-07-18} \cite{openai2022chatgpt} as the LLM and were conducted in March 2025.
\paragraph{Evaluation of Reasoning with LLM} As shown in Table \ref{tab:4_scenarios_results}, we list the problem scale for each scenario, including the number of action predicates \textbf{$|\mathcal{Q}|$}, objects \textbf{$|\mathcal{O}|$}, and full action space size $|\mathcal{A}|$. We also provide oracle heuristics results for all tasks, showing optimal path length $|p^*|$ and the corresponding pruned action space size $|\mathcal{A}^*|$. The results indicate that LLM's reasoning for the pruned action space is accurate, particularly in \textit{RoboWaiter} and \textit{VirtualHome}. Feedback enriches the action space, enhancing SR across all scenarios. The pruned action space $|\mathcal{A}^-|$ and path lengths $|\hat{p}|$ are comparable to the oracle. However, optimal path length does not necessarily imply the smallest action space, as seen in \textit{VirtualHome}. Compared to the baseline OBTEA (B-Times), HBTP significantly reduces planning time, especially in complex scenarios. In the \textit{RobotHow}, OBTEA takes over an hour, while HBTP generates BT in under a second through action space pruning and heuristics.


\section{Related Work}
\paragraph{BT Generation}

Many works have focused on automatically generating BTs to perform tasks, such as evolutionary computing \cite{neupane2019learning,colledanchise2019learning,lim2010evolving}, reinforcement learning \cite{banerjee2018autonomous,pereira2015framework}, imitation learning \cite{french2019learning}, MCTS \cite{scheide2021behavior}, and formal synthesis \cite{li2021reactive,tadewos2022specificationguided,neupane2023designing}. Recently, some work directly generate BT using LLMs \cite{lykov2023llmbrain,lykov2023llmmars}. However, the above methods either require complex environment modeling or cannot guarantee the reliability of the behavior tree. In contrast, BT planning \cite{cai2021bt,chen2024integrating} based on STRIPS-style modeling \cite{fikes1971strips} not only offers intuitive environment modeling but also ensures the reliability and robustness of the generated BTs.




\paragraph{Heuristics for Planning} Beyond traditional search heuristics that rely on domain knowledge \cite{ouessai2020improving, liang2011pruning, gutierrez2021meta, hoffmann2006conformant}, neural networks have recently been employed to learn heuristic functions for planning \cite{chen2024learning, hazra2023saycanpay, papagiannis2022pruning, gutierrez2021meta, ramirez2016heuristics}, though their training is often time-consuming. However, in the context of BT planning, no widely effective heuristic methods have been proposed, to the best of our knowledge.






\paragraph{LLM for Planning} Recent advancements in LLM-based planning include ProgPrompt \cite{singh2022progprompt}, PlanBench \cite{valmeekam2024planbench}, and Voyager \cite{wang2023voyager}. A common approach is to generate action sequences through LLMs \cite{song2023llmplanner,liu2023llm,ahn2022can,chen2023robogpt}. However, plans generated directly by LLMs often lack guarantees of correctness and fall short in terms of responsiveness and robustness during execution. While some methods leverage LLMs for planning heuristics \cite{hazra2023saycanpay,valmeekam2023planning}, they frequently require fine-tuning on specific datasets. In contrast, our task-specific pruning and heuristic approach for BT planning is simpler and more applicable in real-world scenarios.


\section{CONCLUSIONS}
This paper proposes a reliable and efficient BT planning approach for autonomous robots using task-specific reasoning from LLMs. By leveraging heuristic paths from LLMs, along with action space pruning and reflective feedback, HBTP significantly enhances planning efficiency while ensuring BT reliability in complex scenarios. Formal analysis and validation on four datasets confirm its efficacy in real-world service robot applications.

\addtolength{\textheight}{-0cm}
\newpage
\bibliographystyle{IEEEtran}
\bibliography{IEEEabrv,mybib}

 \onecolumn

\appendix

\section{Proof of Proposition \ref{pro:optimality}}

\setcounter{proposition}{0}
\begin{proposition}
	Given a task $<s_0,g>$ and a predicted optimal path $\hat{p}$, if $I(\hat{p},a)\leq I(p^*,a), \forall a\in\mathcal{A}$ and $\alpha \geq 1$, we have $\forall p \in \mathcal{P}(s_0,g), h^\alpha (p^*,\hat{p}) \leq h^\alpha (p,\hat{p})$.
\end{proposition}
\begin{proof}
Because the $p^*$ is the optimal path, we have $\forall p \in \mathcal{P}(s_0,g), D(p^*) \leq D(p)$. In order to simplify the expression, we introduce a virtual path $p'$, where $\forall a \in \mathcal{A}(\hat{p})$, if $I(p,a)>I(\hat{p},a)$, we replace the number of $I(p,a)-I(\hat{p},a)$ actions with virtual actions $a_1,a_2,...,a_k, k=I(p,a)-I(\hat{p},a)$. These actions have the same action model and cost function as $a$, but just in different names. Therefore, we have $\forall a \in \mathcal{A}(\hat{p}), I(p,a)\leq I(\hat{p},a)$, so for any $p\in \mathcal{P}(s_0,g)$:
    \begin{align}
        h^\alpha(p,\hat{p}) &= \sum_{a\in \mathcal{A}-\mathcal{A}(\hat{p})} I(p,a)D(a) + \frac{1}{\alpha} \sum_{a\in \mathcal{A}(\hat{p})} I(p,a)D(a) \label{eq:1} \\
        &= \sum_{a\in \mathcal{A}} I(p,a)D(a) - \sum_{a\in \mathcal{A}(\hat{p})} I(p,a)D(a) + \frac{1}{\alpha} \sum_{a\in \mathcal{A}(\hat{p})} I(p,a)D(a) \\
        &= D(p) - \sum_{a\in \mathcal{A}(\hat{p})} I(p,a)D(a) + \frac{1}{\alpha} \sum_{a\in \mathcal{A}(\hat{p})} I(p,a)D(a) \\
        &= D(p) + (\frac{1}{\alpha} - 1) \sum_{a\in \mathcal{A}(\hat{p})} I(p,a)D(a)\\
    \end{align}
Then, we have: \begin{align}
        h^\alpha(p^*,\hat{p}) - h^\alpha(p,\hat{p}) &= D(p^*)-D(p) + (\frac{1}{\alpha} - 1) \sum_{a\in \mathcal{A}(\hat{p})} \left[ I(p^*,a) - I(p,a) \right] D(a)  \label{eq:1}
    \end{align}
As we know:
\begin{itemize}
    \item $D(p^*)-D(p)\leq 0$
    \item $\alpha \geq 1$ so $\frac{1}{\alpha} - 1 \leq 0$
    \item $\forall a\in \mathcal{A}(\hat{p}), I(p^*,a) \geq  I(\hat{p},a) \geq I(p,a)$ and $D(a)\geq 0$, \\ so $\sum_{a\in \mathcal{A}(\hat{p})} \left[ I(p^*,a) - I(p,a) \right] D(a)\geq 0$
\end{itemize}
Therefore, we have $h^\alpha (p^*,\hat{p}) \leq h^\alpha (p,\hat{p})$, which proofs the proposition.
\end{proof}


	





\section{Datasets and Environments}
\subsection{Four Scenarios}
 In this section, we describe four different scenarios used for testing and validating our approaches. These scenarios provide a variety of environments and tasks that range from simple to complex, helping to evaluate the performance and adaptability of the models.

\begin{itemize} %
\item \textbf{RoboWaiter:} A digitally twinned café environment simulator, where a humanoid robot acts as a waiter equipped with 21 active joints for intricate movement control (see Figure~\ref{fig:robowaiter}). 
\item \textbf{VirtualHome:} A household activities simulation platform designed to model complex interactions within a home setting, such as cooking, cleaning, and other domestic tasks (see Figure~\ref{fig:virtualhome}).
\item \textbf{RobotHow-Small:} A simplified version of the RobotHow dataset, focusing on basic daily tasks, making it easier to test and validate BT planning algorithms.
\item \textbf{RobotHow:} A comprehensive and complex dataset encompassing a wide range of everyday activities and challenging tasks, aimed at pushing the boundaries of robotic task planning and execution(see Figure~\ref{fig:robothow}).
\end{itemize}  

The table \ref{table:action_predicates} summarizes the action predicates available for different scenes in various robotics and virtual environments. Each scene supports a unique set of actions that define the capabilities and interactions possible within that environment. The RoboWaiter scene focuses on tasks related to service and movement, while the VirtualHome scene includes actions pertinent to a home setting. The RobotHow-Small and RobotHow scenes encompass a broader range of actions, indicating their suitability for more complex or varied tasks.

\begin{table*}[h]
	\centering
        \normalsize
	\caption{Action predicates for different scenarios}
	\label{table:action_predicates}
	\begin{tabular}{>{\bfseries}c c}
		\toprule
		Scenarios & \textbf{Included Action Predicates} \\
		\midrule
		\textbf{RoboWaiter} & \constant{Clean}, \constant{Make}, \constant{MoveTo}, \constant{PickUp}, \constant{PutDown}, \constant{Turn} \\
		\midrule
		\multirow{2}{*}{\textbf{VirtualHome}} & \constant{Close}, \constant{LeftGrab}, \constant{LeftPut}, \constant{LeftPutIn}, \constant{Open}, \constant{RightGrab}, \\
		& \constant{RightPut}, \constant{RightPutIn}, \constant{SwitchOff}, \constant{SwitchOn}, \constant{Walk} \\
		\midrule
		\multirow{3}{*}{\textbf{RobotHow-Small}} & \constant{Walk}, \constant{LeftGrab}, \constant{RightGrab}, \constant{LeftPut}, \constant{RightPut},  \\
		& \constant{Open}, \constant{Close}, \constant{LeftPutIn}, \constant{RightPutIn}, \\
        \midrule
		\multirow{2}{*}{\textbf{RobotHow}} & \constant{SwitchOff}, \constant{SwitchOn}, \constant{PlugIn}, \constant{PlugOut}, \\
		& \constant{Cut}, \constant{Wash}, \constant{Wipe} \\
		\bottomrule
	\end{tabular}
\end{table*}

\begin{figure*}[t]
    \centering
    \begin{subfigure}[b]{0.31\textwidth}
        \centering
        \includegraphics[width=\textwidth, height=0.3\textheight, keepaspectratio]{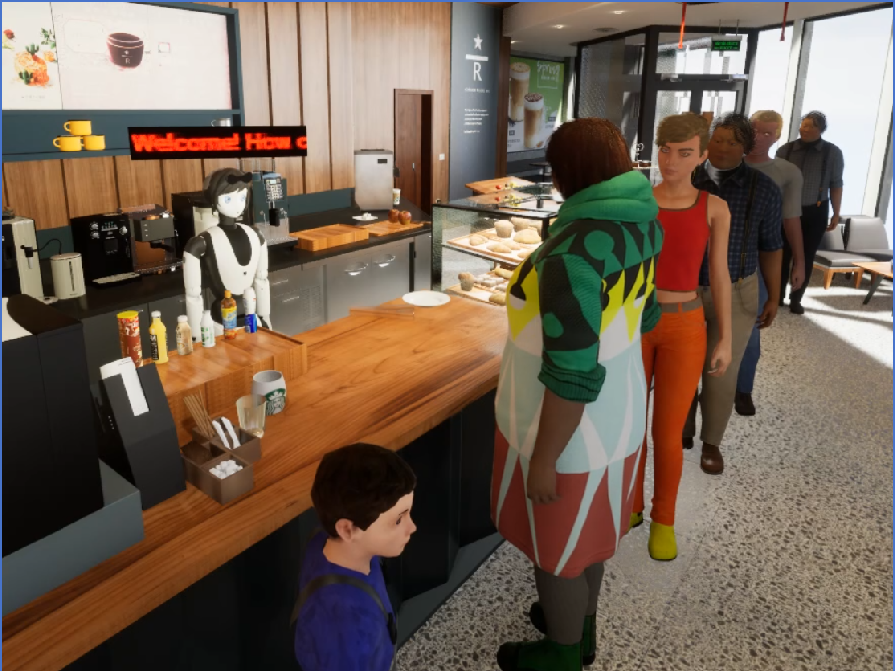}
        \caption{RoboWaiter}
        \label{fig:robowaiter}
    \end{subfigure}
    \hfill
    \begin{subfigure}[b]{0.31\textwidth}
        \centering
        \includegraphics[width=\textwidth, height=0.3\textheight, keepaspectratio]{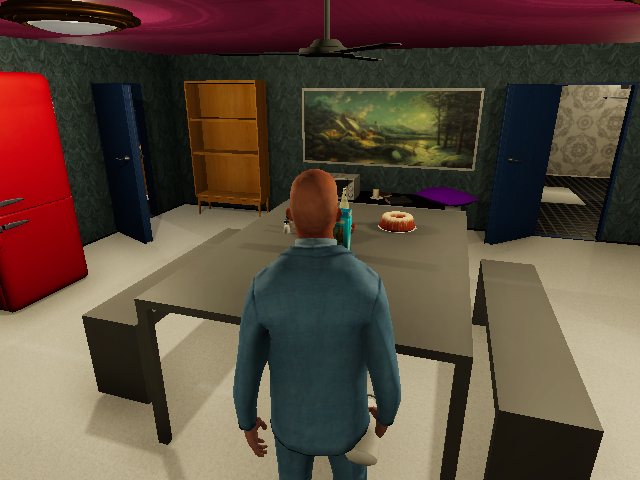}
        \caption{VirtualHome}
        \label{fig:virtualhome}
    \end{subfigure}
    \hfill
    \begin{subfigure}[b]{0.31\textwidth}
        \centering
        \includegraphics[width=\textwidth, height=0.6\textheight, keepaspectratio]{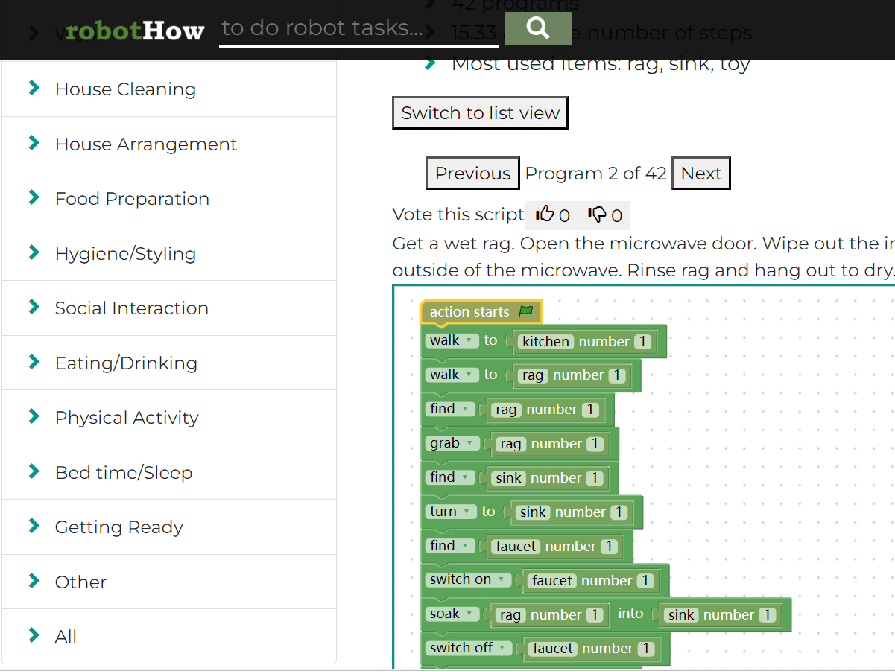}
        \caption{RobotHow}
        \label{fig:robothow}
    \end{subfigure}
    \caption{Illustrations of scenarios}
    \label{fig:scenarios}
\end{figure*}

\subsection{Datasets}

The table above (Table \ref{table:performance_metrics}) illustrates the average optimal path length and the average inclusion of specific action predicates across different datasets. Each dataset contains 50 records and is categorized as Easy, Medium, or Hard, demonstrating varying levels of complexity. In Figure~\ref{fig:bt}, we showcase instances of behavior trees derived from common-sense reasoning employing LLM and HBTP across three distinct difficulty levels of datasets.

\begin{itemize}
	\item \textbf{Easy:} The Easy dataset consists solely of single-task data, such as \constant{IsOn\_apple\_table} and \constant{IsOpen\_microwave}.
	
	\item \textbf{Medium:} The Medium dataset includes both single-task and dual-task data, and features tasks that require tools not specified in the goal. For instance, \constant{IsClean\_sofa} necessitates the use of a \constant{rag}, \constant{IsClean\_apple} requires moving to a faucet and switching it on, and \constant{IsCut\_breadslice} involves using a \constant{kitchenknife}. Additionally, actions like \constant{IsSwitchedOn} require plugging in appliances if not already plugged in, and tasks become unmanageable when both hands are occupied. These complexities present significant challenges to models' commonsense reasoning abilities.
	
	\item \textbf{Hard:} The Hard dataset builds on the Medium dataset by incorporating triple-task data, thus increasing the length of the optimal action sequence and testing the model's ability to solve long-sequence problems. For instance, the task \constant{IsClean\_peach \& IsClean\_desk \& IsOn\_peach\_desk} requires the following optimal sequence of actions: \constant{Walk\_peach, RightGrab\_peach, Walk\_faucet, SwitchOn\_faucet, Wash\_peach, Walk\_rag, LeftGrab\_rag, Walk\_desk, RightPut\_peach\_desk, Wipe\_desk}. This progression significantly challenges the model's capability to handle extended action sequences efficiently.
\end{itemize}

\begin{table*}[h]
	\centering
        \normalsize
	\caption{Average optimal path length and action predicates for different datasets}
	\label{table:performance_metrics}
	\begin{tabular}{>{\bfseries}l ccccccccccc}
		\toprule
		\textbf{} & \textbf{Record Count} & \textbf{$|p^*|$} & \textbf{PlugIn} & \textbf{Open} & \textbf{Close} & \textbf{Cut} & \textbf{Wipe} & \textbf{Wash} \\
		\midrule
		\textbf{Easy} & 50 & 3.1 & 0.3 & 0.3 & 0.2 & 0 & 0 & 0 \\
		\textbf{Medium} & 50 & 4.5 & 0.3 & 0.3 & 0.05 & 0.15 & 0.1 & 0.1 \\
		\textbf{Hard} & 50 & 8.1 & 0.2 & 0.2 & 0.2 & 0.3 & 0.3 & 0.3 \\
		\bottomrule
	\end{tabular}
\end{table*}

\begin{figure*}[t]
	\centering
	\begin{subfigure}[t]{0.35\textwidth}
		\includegraphics[width=\textwidth]{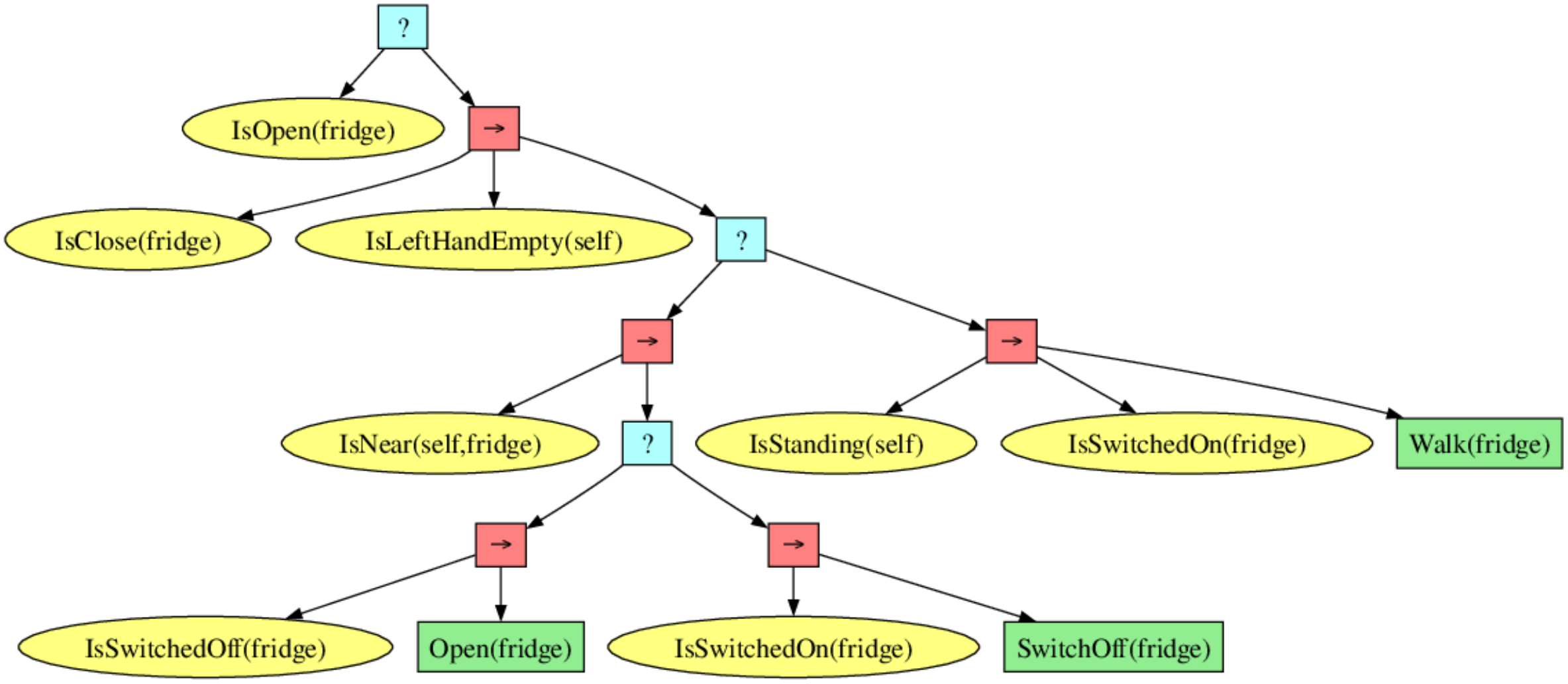}
		\caption{Easy}
		\label{fig:NA}
	\end{subfigure}
	\begin{subfigure}[t]{0.64\textwidth}
		\includegraphics[width=\textwidth]{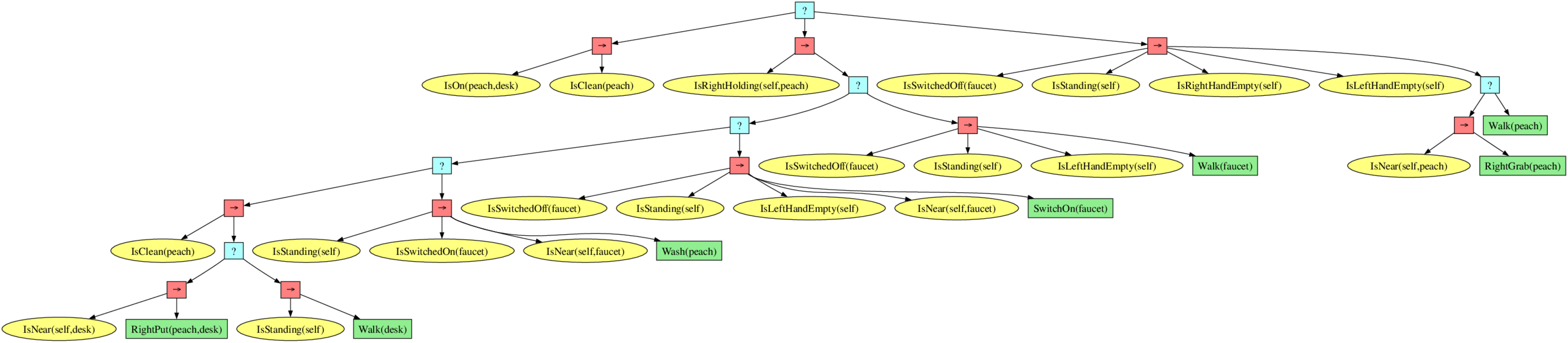}
		\caption{Medium}
		\label{fig:NP}
	\end{subfigure}
	\hfill
	\begin{subfigure}[t]{0.99\textwidth}
		\includegraphics[width=\textwidth]{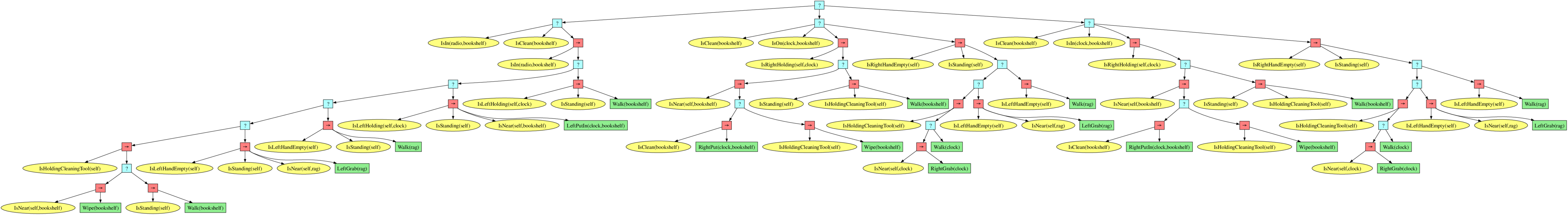}
		\caption{Hard}
		\label{fig:NR}
	\end{subfigure}
	\hfill
	\caption{BT generated based on three difficulty levels of datasets}
	\label{fig:bt}
\end{figure*}

\section{Prompt}

\definecolor{myblue}{RGB}{64,114,196}
\definecolor{mylightblue}{RGB}{243,246,251}
\definecolor{mygreen}{RGB}{38,180,41}
\definecolor{myorange}{RGB}{233,76,5}
\definecolor{blue}{RGB}{27,102,253}
\begin{tcolorbox}[colback=mylightblue,colframe=myblue,title=Prompt, breakable]
    \textbf{[Condition Predicates]} \\
    \textcolor{blue}{IsNear}\_\textcolor{blue}{self}\_\textcolor{myorange}{<ALL>} \\
    \textcolor{blue}{IsOn}\_\textcolor{myorange}{<GRABBABLE>}\_\textcolor{myorange}{<SURFACES>} \\
    ... \\ 
    \textcolor{blue}{IsCut}\_\textcolor{myorange}{<CUTABLE>} \\
    

    \textbf{[Action Predicates]} \\
    \textcolor{blue}{Walk}\_\textcolor{myorange}{<ALL>} \\
    \textcolor{blue}{RightGrab}\_\textcolor{myorange}{<GRABBABLE>} \\
    ... \\
    \textcolor{blue}{Wash}\_\textcolor{myorange}{<WASHABLE>} \\

    
    \textbf{[Objects]} \\
    \textcolor{myorange}{<GRABBABLE>} = [`\textcolor{mygreen}{sundae}', `\textcolor{mygreen}{toothpaste}', ...] \\
    \textcolor{myorange}{<SURFACES>} = [`\textcolor{mygreen}{kitchentable}', `\textcolor{mygreen}{towelrack}', ...] \\
    ... \\
    \textcolor{myorange}{<WASHABLE>} = [`\textcolor{mygreen}{apple}', `\textcolor{mygreen}{bananas}', ...] \\
    \textcolor{myorange}{<ALL>} = \textcolor{myorange}{<GRABBABLE>} + \textcolor{myorange}{<SURFACES>} + \textcolor{myorange}{<CONTAINERS>} + \textcolor{myorange}{<HAS\_SWITCH>} \\ + \textcolor{myorange}{<HAS\_PLUG>} + \textcolor{myorange}{<CUTABLE>} + \textcolor{myorange}{<WASHABLE>} \\

    \textbf{[Few-shot Demonstrations]}\\
    Goals: \textcolor{blue}{IsSwitchedOn}\_\textcolor{myorange}{candle} \\
    Heuristic Path: \textcolor{blue}{Walk}\_\textcolor{myorange}{candle},  \textcolor{blue}{SwitchOn}\_\textcolor{myorange}{candle} \\
    Relevant Action Predicates: \textcolor{blue}{Walk}, \textcolor{blue}{SwitchOn} \\
    Relevant Objects: \textcolor{myorange}{candle} \\
    
    Goals: \textcolor{blue}{IsClean}\_\textcolor{myorange}{peach} \& \textcolor{blue}{IsIn}\_\textcolor{myorange}{peach}\_\textcolor{myorange}{<fridge>} \\
    Heuristic Path: \textcolor{blue}{Walk}\_\textcolor{myorange}{peach}, \textcolor{blue}{RightGrab}\_\textcolor{myorange}{peach}, \textcolor{blue}{Walk}\_\textcolor{myorange}{faucet}, \textcolor{blue}{SwitchOn}\_\textcolor{myorange}{faucet}, \textcolor{blue}{Wash}\_\textcolor{myorange}{peach}, \textcolor{blue}{Walk}\_\textcolor{myorange}{fridge}, \textcolor{blue}{Open}\_\textcolor{myorange}{fridge}, \textcolor{blue}{PlugIn}\_\textcolor{myorange}{fridge}, \textcolor{blue}{RightPutIn}\_\textcolor{myorange}{peach}\_\textcolor{myorange}{fridge} \\
    Relevant Action Predicates: \textcolor{blue}{Walk}, \textcolor{blue}{RightGrab}, \textcolor{blue}{SwitchOn}, \textcolor{blue}{Wash}, \textcolor{blue}{Open}, \textcolor{blue}{PlugIn}, \textcolor{blue}{RightPutIn}  \\
    Relevant Objects: \textcolor{myorange}{peach}, \textcolor{myorange}{faucet}, \textcolor{myorange}{fridge}  \\
    ... \\

    \textbf{[System]} \\ \
    [Condition Predicates] Lists all predicates representing conditions and their optional parameter sets. \\ \
    [Action Predicates] Lists all the actions, specifying their associated costs in parentheses. \\ \
    [Objects] Lists all parameter sets. \\ \
    [Example] Illustrates mappings from goals to Optimal Actions, Relevant Action Predicates, and Relevant Objects, which are essential for completing the tasks outlined in the instructions. \
    \begin{itemize}
        \item Optimal Actions: The sequence of actions with the lowest total cost to achieve the goals. \par
        \item Relevant Action Predicates: Action predicates representing the actions required to achieve the goals. \par
        \item Relevant Objects: Nouns representing all items or entities involved in accomplishing the goals. \par
    \end{itemize}
    \begin{enumerate}
        \item Your task is to analyze the given goal to identify the optimal actions, Relevant action predicates, and Relevant objects necessary for achieving the goals. The goal is presented in first-order logic, consisting of [Condition Predicates]. \par
        \item List all actions needed to accomplish these goals. Begin with `Optimal Actions:' followed by a comma-separated list of actions, using an underscore between the verb and the object. Ensure that the action sequence minimizes the total cost. \par
        \item Identify the essential action predicates used in these actions. These should only be the verbs representing each action. Begin with `Relevant Action Predicates:', followed by a comma-separated list of these verbs. \par
        \item List all crucial objects used in these actions. These should only include the nouns representing items or entities interacted with. Begin with `Relevant Objects:', followed by a comma-separated list of these nouns. \par
        \item The actions, predicates, and objects should come from the provided lists above. If an item does not exist, replace it with the closest available match. \par
        \item Refer directly to the examples in [Few-shot Demonstrations], using `Heuristic Path:', `Relevant Action Predicates:', and `Relevant Objects:' in the given format. Exclude any additional explanations, strictly following the example format without extra headings or line breaks. \par
    \end{enumerate}
    
\end{tcolorbox}

\begin{figure}[H]
	\centering
	\caption{Prompts input into LLMs for reasoning.}
	\label{box:prompt}
\end{figure}

\end{document}